\documentclass{article}

\usepackage{graphicx} 
\usepackage{subfigure}

\usepackage{natbib}

\usepackage{amsmath}

\usepackage{algorithm}
\usepackage{algorithmic}

\usepackage{hyperref}
\usepackage{fullpage}
\usepackage{natbib}

\usepackage{amsthm,amssymb,amsmath}

\newtheorem{definition}{Definition}
\newtheorem{lemma}[definition]{Lemma}
\newtheorem{theorem}[definition]{Theorem}

\newcommand{\norm}[1]{\left\|#1\right\|}
\newcommand{\ip}[2]{\left\langle {#1}, {#2} \right\rangle}  
\newcommand{\R}{\mathbb{R}}  

\newcommand{\mc}[1]{\mathcal{#1}}
\newcommand{\emp}[1]{\widehat{#1}}

\DeclareMathOperator{\Regret}{Regret}    
\DeclareMathOperator*{\E}{\mathbf{E}}    
\DeclareMathOperator*{\argmin}{arg\,min}

\renewcommand{\O}{\mathcal{O}}
\newcommand{\zk}[1]{}
\newcommand{\satyen}[1]{}
\newcommand{\tl}[1]{}
\newcommand{\david}[1]{}

\newcommand{\sk}{{k_0}}
\newcommand{\wk}{{k_1}}

\begin{document}

\title{Adaptive Feature Selection: Computationally Efficient Online Sparse Linear Regression under RIP}

\author{Satyen Kale\thanks{This work was done while the author was at Yahoo Research, New York.} \\
Google Research, New York \\
\texttt{satyenkale@google.com}
\and
Zohar Karnin$^*$\\
Amazon, New York\\
\texttt{zkarnin@gmail.com}
\and
Tengyuan Liang$^*$\\
University of Chicago \\
Booth School of Business\\
\texttt{Tengyuan.Liang@chicagobooth.edu}
\and
D\'avid P\'al\\
Yahoo Research, New York\\
\texttt{dpal@yahoo-inc.com}
}

\date{}
\maketitle

\begin{abstract}
Online sparse linear regression is an online problem where an algorithm
repeatedly chooses a subset of coordinates to observe in an adversarially chosen feature vector, makes a real-valued prediction, receives the true label, and incurs the squared loss. The goal is to design an online learning algorithm with sublinear regret to the best sparse linear predictor in hindsight. Without any assumptions, this problem is known to be computationally intractable. In this paper, we make the assumption that data matrix satisfies restricted isometry property, and show that this assumption leads to computationally efficient algorithms with sublinear regret for two variants of the problem. In the first variant, the true label is generated according to a sparse linear model with additive Gaussian noise. In the second, the true label is chosen adversarially.
\end{abstract}

\section{Introduction}

In modern real-world sequential prediction problems, samples are typically high
dimensional, and construction of the features may itself be a computationally
intensive task. Therefore in sequential prediction, due to the computation and
resource constraints, it is preferable to design algorithms that compute only a
limited number of features for each new data example. One example of this
situation, from \citep{cesa2011efficient}, is medical diagnosis of a disease, in
which each feature is the result of a medical test on the patient. Since it is
undesirable to subject a patient to a battery of medical tests, we would like to
adaptively design diagnostic procedures that rely on only a few, highly
informative tests.

\emph{Online sparse linear regression} (OSLR) is a sequential prediction problem in which an algorithm is allowed to see only a small subset of coordinates of each feature vector. The problem is parameterized by 3 positive integers: $d$, the dimension of the feature vectors, $k$, the sparsity of the linear regressors we compare the algorithm's performance to, and $\sk$, a budget on the number of features that can be queried in each round by the algorithm. Generally we have $k \ll d$ and $\sk \geq k$ but not significantly larger (our algorithms need\footnote{In this paper, we use the $\tilde{O}(\cdot)$ notation to suppress factors that are polylogarithmic in the natural parameters of the problem.} $\sk = \tilde{O}(k)$).

In the OSLR problem, the algorithm makes predictions over a sequence of $T$ rounds. In each round $t$, nature chooses a feature vector $x_t
\in \R^d$, the algorithm chooses a subset of $\{1,2,\dots,d\}$ of size at most $k'$ and observes the corresponding coordinates of the feature
vector. It then makes a prediction $\emp{y}_t \in \R$ based on the
observed features, observes the true label $y_t$, and suffers loss $(y_t -
\emp{y}_t)^2$. The goal of the learner is to make the cumulative loss comparable
to that of the best $k$-sparse linear predictor $w$ in hindsight. The
performance of the online learner is measured by the {\em regret}, which is
defined as the difference between the two losses:
$$
\Regret_T = \sum_{t=1}^T \left(y_t - \emp{y}_t\right)^2 - \min_{w:\ \|w\|_0 \leq k} \sum_{t=1}^T \left(y_t - \ip{x_t}{w} \right)^2 \, .
$$
The goal is to construct algorithms that enjoy regret that is sub-linear in $T$, the total number of rounds. A sub-linear regret implies that in the asymptotic sense, the average per-round loss of the algorithm approaches the average per-round loss of the best $k$-sparse linear predictor.

Sparse regression is in general a computationally hard problem. In particular,
given $k$, $x_1, x_2, \dots, x_T$ and $y_1, y_2, \dots, y_T$ as inputs, the
offline problem of finding a $k$-sparse $w$ that minimizes the error
$\sum_{t=1}^T (y_t - \ip{x_t}{w})^2$ does not admit a polynomial time algorithm
under standard complexity assumptions~\cite{foster2015variable}. This hardness
persists even under the assumption that there exists a $k$-sparse $w^*$ such
that $y_t = \ip{x_t}{w^*}$ for all $t$.  Furthermore, the computational hardness
is present even when the solution is required to be only $\widetilde{\O}(k)$-sparse
solution and has to minimize the error only approximately;
see~\cite{foster2015variable} for details. The hardness result was extended to
online sparse regression by~\cite{foster2016online}.  They showed that for all
$\delta > 0$ there exists no polynomial-time algorithm with regret
$\O(T^{1-\delta})$ unless $NP \subseteq BPP$.

\citet{foster2016online} posed the open question of what additional assumptions
can be made on the data to make the problem tractable. In this paper, we answer
this open question by providing efficient algorithms with sublinear regret under
the assumption that the matrix of feature vectors satisfies the \emph{restricted
isometry property} (RIP)~\citep{candes2005decoding}. It has been shown that if
RIP holds and there exists a sparse linear predictor $w^*$ such that $y_t =
\ip{x_t}{w^*} + \eta_t$ where $\eta_t$ is independent noise, the offline
sparse linear regression problem admits computationally efficient algorithms,
e.g., \cite{candes2007dantzig}. RIP and related Restricted Eigenvalue Condition \citep{bickel2009simultaneous} have been widely used as a standard assumption for theoretical analysis in the compressive sensing and sparse regression literature, in the offline case. In the online setting, it is natural to ask whether sparse regression avoids the computational difficulty under an appropriate form of the RIP condition. In this paper, we answer this question in a positive way, both in the realizable setting and in the agnostic setting. As a by-product, we resolve the adaptive feature
selection problem as the efficient algorithms we propose in this paper
adaptively choose a different ``sparse'' subset of features to query at each round. This is
closely related to attribute-efficient learning (see discussion in
Section~\ref{sec:related-work}) and online model selection.

\subsection{Summary of Results}

We design polynomial-time algorithms for online sparse linear regression for two
models for the sequence $(x_1, y_1), (x_2, y_2), \dots, (x_T, y_T)$. The first
model is called the \emph{realizable} and the second is called \emph{agnostic}.
In both models, we assume that, after proper normalization, for all large enough
$t$, the matrix $X_t$ formed from the first $t$ feature vectors $x_1, x_2,
\dots, x_t$ satisfies the restricted isometry property. The two models differ in
the assumptions on $y_t$. The realizable model assumes that $y_t = \ip{x_t}{w^*} +
\eta_t$ where $w^*$ is $k$-sparse and $\eta_t$ is an independent noise. In the
agnostic model, $y_t$ can be arbitrary, and therefore, the regret bounds we obtain are worse than in the realizable setting. The models and corresponding algorithms are presented in Sections~\ref{section:realizable}~and~\ref{section:agnostic} respectively. Interestingly enough, the algorithms and their corresponding analyses are completely different in the realizable and agnostic case.

Our algorithms allow for somewhat more flexibility than the problem definition:
they are designed to work with a budget $\sk$ on the number of features that can
be queried that may be larger than the sparsity parameter $k$ of the comparator.
The regret bounds we derive improve with increasing values of $\sk$. In the case
when $\sk \approx k$, the dependence on $d$ in the regret bounds is polynomial,
as can be expected in limited feedback settings (this is analogous to polynomial
dependence on $d$ in \emph{bandit} settings). In the extreme case when $\sk = d$,
i.e. we have access to \emph{all} the features, the dependence on the dimension
$d$ in the regret bounds we prove is only \emph{logarithmic}. The interpretation
is that if we have full access to the features, but the goal is to compete with
just $k$ sparse linear regressors, then the number of data points that need to
be seen to achieve good predictive accuracy has only logarithmic dependence on
$d$. This is analogous to the (offline) compressed sensing setting where the
sample complexity bounds, under RIP, only depend logarithmically on $d$.

A major building block in the solution for the realizable setting
(Section~\ref{section:realizable}) consists of identifying the best $k$-sparse
linear predictor for the past data at any round in the prediction problem. This
is done by solving a sparse regression problem on the observed data. The
solution of this problem cannot be obtained by a simple application of say, the
Dantzig selector \citep{candes2007dantzig} since we do not observe the data
matrix $X$, but rather a subsample of its entries. Our algorithm is a variant of
the Dantzig selector that incorporates random sampling into the optimization,
and computes a near-optimal solution by solving a linear program. The resulting
algorithm has a regret bound of $\widetilde{\O}(\log T)$. This bound has optimal dependence on $T$, since even in the full information setting where all features are observed there is a lower bound of $\Omega(\log T)$ \citep{hazan-kale-2014}.

The algorithm for the agnostic setting relies on the theory of submodular
optimization. The analysis in \citep{boutsidis2015greedy} shows that the RIP
assumption implies that the set function defined as the minimum loss achievable
by a linear regressor restricted to the set in question satisfies a property
called {\em weak supermodularity}. Weak supermodularity is a relaxation of
standard supermodularity that is still strong enough to show performance bounds
for the standard greedy feature selection algorithm for solving the sparse
regression problem. We then employ a technique developed by
\citet{golovin-streeter} to construct an online learning algorithm that mimics
the greedy feature selection algorithm. The resulting algorithm has a regret
bound of $\widetilde{\O}(T^{2/3})$. It is unclear if this bound has the optimal dependence on $T$: it is easy to prove a lower bound of $\Omega(\sqrt{T})$ on the regret using standard arguments for the multiarmed bandit problem.

\subsection{Related work}
\label{sec:related-work}

A related setting is attribute-efficient learning
\citep{cesa2011efficient,HK,KS}. This is a batch learning problem in which the
examples are generated i.i.d., and the goal is to simply output a linear
regressor using only a limited number of features per example with bounded
excess risk compared to the optimal linear regressor, when given {\em full
access} to the features at test time. Since the goal is not prediction but
simply computing the optimal linear regressor, efficient algorithms exist and
have been developed by the aforementioned papers.

Without any assumptions, only inefficient algorithms for the online sparse
linear regression problem are known \citet{andras,foster2016online}.
\citet{satyen-open-problem} posed the open question of whether it is possible to
design an efficient algorithm for the problem with a sublinear regret bound.
This question was answered in the negative by \citet{foster2016online}, who
showed that efficiency can only be obtained under additional assumptions on the
data. This paper shows that the RIP assumption yields tractability in the online
setting just as it does in the batch setting.


In the realizable setting, the linear program at the heart of the algorithm is
motivated from Dantzig selection \cite{candes2007dantzig} and error-in-variable
regression \cite{rosenbaum2010sparse,belloni2016linear}.
The problem of finding the best sparse linear predictor when only a sample of
the entries in the data matrix is available is also discussed by
\citet{belloni2016linear} (see also the references therein). In fact, these
papers solve a more general problem where we observe a matrix $Z$ rather than
$X$ that is an unbiased estimator of $X$. While we can use their results in a
black-box manner, they are tailored for the setting where the variance of each
$Z_{ij}$ is constant and it is difficult to obtain the exact dependence on this
variance in their bounds. In our setting, this variance can be linear in the
dimension of the feature vectors, and hence we wish to control the dependence on
the variance in the bounds. Thus, we use an algorithm that is similar to the one
in \cite{belloni2016linear}, and provide an analysis for it (in the
appendix). As an added bonus, our algorithm results in solving a linear program
rather than a conic or general convex program, hence admits a solution that is
more computationally efficient.

In the agnostic setting,  the computationally efficient algorithm we propose is motivated from (online) supermodular optimization \citep{natarajan1995sparse, boutsidis2015greedy, golovin-streeter}. The algorithm is computationally efficient and enjoys sublinear regret under an RIP-like condition, as we will show in Section~\ref{section:agnostic}. This result can be contrasted with the known computationally prohibitive algorithms for online sparse linear regression \citep{andras,foster2016online}, and the hardness result without RIP \citep{foster2015variable, foster2016online}.

\subsection{Notation and Preliminaries}

For $d \in \mathbb{N}$, we denote by $[d]$ the set $\{1,2,\dots,d\}$.  For a vector in $x \in \R^d$,
denote by $x(i)$ its $i$-th coordinate. For a subset $S \subseteq [d]$, we use the notation $\R^S$ to indicate the vector space spanned by the coordinate axes indexed by $S$ (i.e. the set of all vectors $w$ supported on the set $S$). For a vector $x \in \R^d$, denote by $x(S) \in \R^{d}$ the projection of $x$ on $\R^S$. That is, the coordinates of $x(S)$ are
$$
x(S)(i) =
\begin{cases}
x(i) & \text{if $i \in S$,} \\
0  & \text{if $i \not \in S$,}
\end{cases}
\qquad \text{for $i=1,2,\dots,d$.}
$$
Let $\ip{u}{v} = \sum_i u(i) \cdot v(i)$ be the inner product of vectors $u$ and $v$. 

For $p \in [0,\infty]$, the $\ell_p$-norm of a vector $x \in \R^d$ is denoted by
$\norm{x}_p$. For $p \in (0, \infty)$, $\|x\|_p = (\sum_i |x_i|^p)^{1/p}$,
$\|x\|_\infty = \max_i |x_i|$, and $\norm{x}_0$ is the number of non-zero
coordinates of $x$.

The following definition will play a key role:

\begin{definition}[Restricted Isometry Property~\cite{candes2007dantzig}]
\label{definition:RIP}
Let $\epsilon\in (0,1)$ and $k \ge 0$.
We say that a matrix $X \in \R^{n \times d}$ satisfies
\emph{restricted isometry property} (RIP) with parameters $(\epsilon, k)$
if for any $w \in \R^d$ with $\norm{w}_0 \le k$ we have
\begin{align*}
(1 - \epsilon) \norm{w}_2 \le \frac{1}{\sqrt{n}} \norm{X w}_2 \le (1 + \epsilon) \norm{w}_2.
\end{align*}
\end{definition}

One can show that RIP holds with overwhelming probability if $n =
\Omega(\epsilon^{-2} k\log(ed/k))$ and each row of the matrix is sampled
independently from an isotropic sub-Gaussian distribution. In the realizable setting, the sub-Gaussian
assumption can be relaxed to incorporate heavy tail distribution via the ``small
ball" analysis introduced in~\citet{mendelson2014learning}, since we only require one-sided lower isometry property.

\subsection{Proper Online Sparse Linear Regression}

We introduce a variant of online sparse regression (OSLR), which we call
\emph{proper online sparse linear regression (POSLR)}. The adjective ``proper''
is to indicate that the algorithm is required to output a weight vector in each
round and its prediction is computed by taking an inner product with the feature
vector.

We assume that there is an underlying sequence $(x_1, y_1), (x_2, y_2), \dots, (x_T, y_T)$
of \emph{labeled examples} in $\R^d \times \R$. In each round $t=1,2,\dots,T$,
the algorithm behaves according to the following protocol:
\begin{enumerate}
\item Choose a vector $w_t \in \R^d$ such that $\norm{w_t}_0 \le k$.
\item Choose $S_t \subseteq [d]$ of size at most $\sk$.
\item Observe $x_t(S_t)$ and $y_t$, and incur loss $(y_t - \ip{x_t}{w_t})^2$.
\end{enumerate}
Essentially, the algorithm makes the prediction $\emp{y}_t := \ip{x_t}{w_t}$ in round $t$. The regret after $T$ rounds of an algorithm with respect to $w \in \R^d$ is
$$
\Regret_T(w) = \sum_{t=1}^T \left(y_t - \ip{x_t}{w_t}\right)^2 - \sum_{t=1}^T \left(y_t - \ip{x_t}{w} \right)^2 \, .
$$
The regret after $T$ rounds of an algorithm with respect to the best $k$-sparse linear regressor is defined as
$$
\Regret_T = \max_{w:\ \|w\|_0 \le k} \Regret_T(w) \, .
$$

Note that any algorithm for POSLR gives rise to an algorithm for OSLR.
Namely, if an algorithm for POSLR chooses $w_t$ and $S_t$, the corresponding
algorithm for OSLR queries the coordinates $S_t \cup \{ i ~: w_t(i) \neq 0 \}$.
The algorithm for OSLR queries at most $\sk + k$ coordinates and has the same
regret as the algorithm for POSLR.

Additionally, POSLR allows parameters settings which do not have
corresponding counterparts in OSLR. Namely, we can consider the sparse
``full information'' setting where $\sk = d$ and $k \ll d$.

We denote by $X_t$ the $t \times d$ matrix of first $t$ unlabeled samples i.e.
rows of $X_t$ are $x_1^T, x_2^T, \dots, x_t^T$.
Similarly, we denote by $Y_t \in \R^t$ the vector of first $t$ labels $y_1, y_2, \dots, y_t$.
We use the shorthand notation $X$, $Y$ for $X_T$ and $Y_T$ respectively.

In order to get computationally efficient algorithms, we assume that that for
all $t \ge t_0$, the matrix $X_t$ satisfies the restricted isometry condition.
The parameter $t_0$ and RIP parameters $k,\epsilon$ will be specified later.

\section{Realizable Model}
\label{section:realizable}

In this section we design an algorithm for POSLR for the realizable model. In
this setting we assume that there is a vector $w^* \in \R^d$ such that
$\norm{w^*}_0 \le k$ and the sequence of labels $y_1, y_2, \dots, y_T$ is
generated according to the linear model
\begin{equation}
\label{equation:stochastic-model}
y_t = \ip{x_t}{w^*} + \eta_t \; ,
\end{equation}
where $\eta_1, \eta_2, \dots, \eta_T$ are independent random variables from
$N(0, \sigma^2)$. We assume that the standard deviation $\sigma$, or an upper
bound of it, is given to the algorithm as input. We assume that $\norm{w^*}_1
\le 1$ and $\norm{x_t}_\infty \le 1$ for all $t$.

For convenience, we use $\eta$ to denote the vector $(\eta_1, \eta_2, \dots,
\eta_T)$ of noise variables.

\subsection{Algorithm}
\label{section:realizable-algorithm}

The algorithm maintains an unbiased estimate $\emp{X}_t$ of the matrix $X_t$.
The rows of $\emp{X}_t$ are vectors $\emp{x}_1^T, \emp{x}_2^T, \dots,
\emp{x}_t^T$ which are unbiased estimates of $x_1^T, x_2^T, \dots, x_t^T$. To
construct the estimates, in each round $t$, the set $S_t \subseteq [d]$ is
chosen uniformly at random from the collection of all subsets of $[d]$ of size
$\sk$. The estimate is
\begin{equation}
\label{equation:unbiased-estimate}
\emp{x}_t = \frac{d}{\sk} \cdot x_t(S_t).
\end{equation}

To compute the predictions of the algorithm, we consider the linear program
\begin{align}
\label{equation:conv.prog}
\begin{aligned}
\text{minimize} \norm{w}_1 \ & \text{s.t.} \
\norm{\frac{1}{t} \emp{X}_t^T \left( Y_t - \emp{X}_t w \right)  + \frac{1}{t} \emp{D}_t w}_{\infty} \\
& \quad \le C \sqrt{ \frac{d \log (td/\delta)}{t \sk} } \left(\sigma + \frac{d}{\sk}  \right).
\end{aligned}
\end{align}
Here, $C > 0$ is a universal constant, and $\delta \in (0,1)$ is the allowed
failure probability. $\emp{D}_t$, defined in
equation~\eqref{equation:diag.bias}, is a diagonal matrix that offsets the bias
on the $\text{diag}(\emp{X}_t^T \emp{X}_t)$.

The linear program~\eqref{equation:conv.prog} is called the Dantzig selector. We
denote its optimal solution by $\emp{w}_{t+1}$. (We define $\emp{w}_1 = 0$.)

Based on $\emp{w}_t$, we construct $\widetilde{w}_t \in \R^d$. Let $|\emp{w}_t(i_1)| \ge
|\emp{w}_t(i_2)| \ge \dots \ge |\emp{w}_t(i_d)|$ be the coordinates
sorted according to the their absolute value, breaking ties according to their
index. Let $\widetilde{S}_t = \{i_1, i_2, \dots, i_k\}$ be the top $k$ coordinates.
We define $\widetilde{w}_t$ as
\begin{equation}
\label{equation:top-k-coordinates}
\widetilde{w}_t = \emp{w}_t(\widetilde{S}_t).
\end{equation}
The actual prediction $w_t$ is either zero if $t \le t_0$ or
$\widetilde{w}_s$ for some $s \le t$ and it gets
updated whenever $t$ is a power of $2$.

\begin{algorithm}[ht]
\caption{Dantzig Selector for POSLR}
\label{algorithm:dantzig}
\begin{algorithmic}[1]
\REQUIRE{$T$, $\sigma$, $t_0$, $k$, $k_0$}
\FOR{$t = 1,2,\dots,T$}
\IF{$t \le t_0$}
\STATE{Predict $w_t = 0$}
\ELSIF{$t$ is a power of $2$}
\STATE{Let $\emp{w}_t$ be the solution of linear program~\eqref{equation:conv.prog}}
\STATE{Compute $\widetilde{w}_t$ according to~\eqref{equation:top-k-coordinates}}
\STATE{Predict $w_t = \widetilde{w}_t$}
\ELSE
\STATE{Predict $w_t = w_{t-1}$}
\ENDIF
\STATE{Let $S_t \subseteq [d]$ be a random subset of size $\sk$ }
\STATE{Observe $x_t(S_t)$ and $y_t$}
\STATE{Construct estimate $\emp{x}_t$ according to~\eqref{equation:unbiased-estimate}}
\STATE{Append $\emp{x}_{t}^T$ to $\emp{X}_{t-1}$ to form $\emp{X}_t \in \R^{t\times d}$}
\ENDFOR
\end{algorithmic}
\end{algorithm}

The algorithm queries at most $k+\sk$ features each round, and the linear
program can be solved in polynomial time using simplex method or interior point
method. The algorithm solves the linear program only $\lceil\log_2 T \rceil$
times by using the same vector in the rounds $2^s,\ldots,2^{s+1}-1$. This lazy
update improves both  the computational aspects of the algorithm and the regret
bound.

\subsection{Main Result}
\label{section:realizable-main}

The main result in this section provides a logarithmic regret bound under the
following assumptions~\footnote{A more precise statement with the exact
dependence on the problem parameters can be found in
the appendix.}
\begin{itemize}
\item The feature vectors have the property that for any $t \geq t_0$, the
matrix $X_t$ satisfies the RIP condition with $(\frac{1}{5}, 3k)$, with $t_0 = \O(k
\log(d)\log(T))$.

\item The underlying POSLR online prediction problem has a sparsity budget of
$k$ and observation budget $\sk$.

\item The model is realizable as defined in
equation~\eqref{equation:stochastic-model} with i.i.d unbiased Gaussian
noise with standard deviation $\sigma = \O(1)$.
\end{itemize}

\begin{theorem} \label{theorem:realizable}
For any $\delta>0$, with probability at least $1-\delta$, Algorithm~\ref{algorithm:dantzig}
satisfies
$$
\Regret_T = \O \left( k^2 \log(d/\delta) (d/\sk)^3 \log(T) \right).
$$
\end{theorem}

The theorem asserts that an $\O(\log T)$ regret bound is efficiently achievable
in the realizable setting. Furthermore when $\sk = \Omega(d)$ the regret scales
as $\log(d)$ meaning that we do not necessarily require $T \geq d$ to obtain a
meaningful result. We note that the complete expression for arbitrary $t_0, \sigma$ is given in \eqref{equation:regret-realizable} in the appendix.

The algorithm can be easily understood via the error-in-variable equation
\begin{align*}
y_t & = \ip{x_t}{w^*} + \eta_t \; , \\
\emp{x}_t & = x_t + \xi_t.
\end{align*}
with $\E[\xi_t] = \E[\emp{x}_t - x_t] = 0$, where the expectation is
taken over random sampling introduced by the algorithm when performing feature
exploration. The learner observes $y_t$ as well as the ``noisy'' feature vector
$\emp{x}_t$, and aims to recover $w^*$.

As mentioned above, we (implicitly) need an unbiased estimator of $X_t^T X_t$.
By taking $\emp{X}_t^T \emp{X}_t$ it is easy to verify that the off-diagonal
entries are indeed unbiased however this is not the case for the diagonal. To
this end we define $D_t \in \R^{d \times d}$ as the diagonal matrix
compensating for the sampling bias on the diagonal elements of $\emp{X}_t^T
\emp{X}_t$
$$
D_t = \left(\frac{d}{\sk}-1\right) \cdot \text{diag} \left(X_t^T X_t \right)
$$
and the estimated bias from the observed data is
\begin{align} \label{equation:diag.bias}
\emp{D}_t = \left(1 - \frac{\sk}{d} \right) \cdot \text{diag} \left(\emp{X}_t^T \emp{X}_t \right).
\end{align}
Therefore, program~\eqref{algorithm:dantzig} can be viewed as Dantzig selector
with plug-in unbiased estimates for $X_t^T Y_t$ and $X_t^T X_t$ using limited
observed features.

\subsection{Sketch of Proof}
The main building block in proving Theorem~\ref{theorem:realizable} is stated in
Lemma~\ref{lemma:dantzig}. It proves that the sequence of solutions $\emp{w}_t$
converges to the optimal response $w^*$ based on which the signal $y_t$ is
created. More accurately, ignoring all second order terms, it shows that
$\|\emp{w}_t - w^*\|_1 \leq \O(1/\sqrt{t})$. In Lemma~\ref{lemma:sparse} we show
that the same applies for the sparse approximation $w_t$ of $\emp{w}_t$. Now,
since $\|x_t\|_\infty \leq 1$ we get that the difference between our response
$\ip{x_t}{w_t}$ and the (almost) optimal response $\ip{x_t}{w^*}$ is bounded by
$1/\sqrt{t}$. Given this, a careful calculation of the difference of losses
leads to a regret bound w.r.t.\ $w^*$. Specifically, an elementary analysis of
the loss expression leads to the equality
$$
\Regret_T(w^*) = \sum_{t=1}^T 2 \eta_t \ip{x_t}{w^* - w_t} + \left(\ip{x_t}{w^* - w_t} \right)^2 
$$
A bound on both summands can clearly be expressed in terms of $|\ip{x_t}{w^* -
w_t}| = \O(1/\sqrt{t})$. The right summand requires a martingale concentration
bound and the left is trivial. For both we obtain a bound of $\O(\log(T))$.

We are now left with two technicalities. The first is that $w^*$ is not
necessarily the empirically optimal response. To this end we provide, in
Lemma~\ref{lem:wstart_reg} in the appendix, a constant
(independent of $T$) bound on the regret of $w^*$ compared to the empirical
optimum. The second technicality is the fact that we do not solve for
$\emp{w}_t$ in every round, but in exponential gaps. This translates to an added
factor of $2$ to the bound $\|w_t - w^*\|_1$ that affects only the constants in
the $\O(\cdot)$ terms.

\begin{lemma}[Estimation Rates]
\label{lemma:dantzig}
Assume that the matrix $X_t \in \R^{t \times d}$ satisfies the RIP condition with $(\epsilon,3k)$ for some
$\epsilon < 1/5$. Let $\emp{w}_{n+1} \in \R^d$ be the optimal solution of program~\eqref{equation:conv.prog}.
With probability at least $1 - \delta$,
\begin{align*}
\norm{\emp{w}_{t+1} - w^*}_2 \le C \cdot \sqrt{ \frac{d}{\sk} \cdot \frac{k \log (d/\delta)}{t} } \left( \sigma + \frac{d}{\sk}  \right) \; , \\
\norm{\emp{w}_{t+1} - w^*}_1 \le C \cdot \sqrt{ \frac{d}{\sk} \frac{k^2 \log(d/\delta)}{t} } \left(\sigma + \frac{d}{\sk}  \right).
\end{align*}
Here $C>0$ is some universal constant and $\sigma$ is the standard deviation
of the noise.
\end{lemma}

Note the $\emp{w}_t$ may not be sparse; it can have many non-zero coordinates
that are small in absolute value. However, we take the top $k$ coordinates of
$\emp{w}_t$ in absolute value. Thanks to the Lemma~\ref{lemma:sparse} below,
we lose only a constant factor $\sqrt{3}$.

\begin{lemma}
\label{lemma:sparse}
Let $\emp{w} \in \R^d$ be an arbitrary vector and let $w^* \in \R^d$ be a
$k$-sparse vector. Let $\widetilde{S} \subseteq [d]$ be the top $k$
coordinates of $\emp{w}$ in absolute value. Then,
$$
\norm{\emp{w}(\widetilde{S}) - w^*}_2 \le \sqrt{3} \norm{\emp{w} - w^*}_2.
$$
\end{lemma}

\section{Agnostic Setting}
\label{section:agnostic}

In this section we focus on the agnostic setting, where we don't impose any
distributional assumption on the sequence. In this setting, there is no ``true''
sparse model, but the learner --- with limited access to features --- is
competing with the best $k$-sparse model defined using full information $\left\{
(x_t, y_t) \right\}_{t=1}^T$.

As before, we do assume that $x_t$ and $y_t$ are bounded. Without loss of
generality, $\norm{x_t}_\infty \le 1$, and $|y_t| \le 1$ for all $t$. Once
again, without any regularity condition on the design matrix,
\citet{foster2016online} have shown that achieving a sub-linear regret
$\O(T^{1-\delta})$ is in general computationally hard, for any constant
$\delta>0$ unless $\text{NP} \subseteq \text{BPP}$.

We give an efficient algorithm that achieves sub-linear regret under the assumption that the design matrix of any (sufficiently long) block of consecutive data points has bounded \emph{restricted condition number},
which we define below:

\begin{definition}[Restricted Condition Number] \label{definition:restricted-condition-number}
Let $k \in \mathbb{N}$ be a sparsity parameter. The {\em restricted condition
number for sparsity $k$} of a matrix $X \in \R^{n \times d}$ is defined
as
\[ \sup_{\substack{v, w:\ \|v\| = \|w\| = 1, \\ \|v\|_0, \|w\|_0 \leq k}} \frac{\|Xv\|}{\|Xw\|}. \]
\end{definition}

It is easy to see that if a matrix $X$ satisfies RIP with parameters $(\epsilon,
k)$, then its restricted condition number for sparsity $k$ is at most
$\frac{1+\epsilon}{1-\epsilon}$. Thus, having bounded restricted condition
number is a weaker requirement than RIP.

We now define the {\em Block Bounded Restricted Condition Number Property} (BBRCNP):

\begin{definition}[Block Bounded Restricted Condition Number Property]
\label{definition:block-BRCNP}
Let $\kappa > 0$ and $k \in \mathbb{N}$. A sequence of feature vectors $x_1,
x_2, \ldots, x_T$ satisfies {\em BBRCNP} with parameters $(\kappa, K)$ if there
is a constant $t_0$ such that for any sequence of consecutive time steps
$\mc{T}$ with $|\mc{T}| \geq t_0$, the restricted condition number for sparsity
$k$ of $X$, the design matrix of the feature vectors $x_t$ for $t \in \mc{T}$,
is at most $\kappa$.
\end{definition}
Note that in the random design setting where $x_t$, for $t \in [T]$, are isotropic sub-Gaussian vectors, $t_0 = O(\log T + k \log d)$ suffices to satisfy BBRCNP with high probability, where the $O(\cdot)$ notation hides a constant depending on $\kappa$.

We assume in this section that the sequence of feature vectors satisfies BBRCNP
with parameters $(\kappa, K)$ for some $K = \O(k \log(T))$ to be defined in the
course of the analysis.

\subsection{Algorithm}
\label{section:adv.alg}

The algorithm in the agnostic setting is of distinct nature from that in the
stochastic setting. Our algorithm is motivated from literature on maximization
of sub-modular set function
\citep{natarajan1995sparse,golovin-streeter,boutsidis2015greedy}. Though the
problem being NP-hard, greedy algorithm on sub-modular maximization provides
provable good approximation ratio. Specifically, \cite{golovin-streeter}
considered online optimization of super/sub-modular set functions using expert
algorithm as sub-routine. \cite{natarajan1995sparse,boutsidis2015greedy} cast
the sparse linear regression as maximization of weakly supermodular function.
We will introduce an algorithm that blends various ideas from referred
literature, to attack the online sparse regression with limited features.

First, let's introduce the notion of a weakly supermodular function.
\begin{definition}
\label{def: weak-sup-modu}
For parameters $k \in \mathbb{N}$ and $\alpha\geq 1$, a set function $g: [d]
\rightarrow \R$ is $(k, \alpha)$-weakly supermodular if for any two sets
$S\subseteq T \subseteq [d]$ with $|T| \leq k$, the following two inequalities
hold:
\begin{enumerate}
	\item {\bf (monotonicity)} $g(T) \leq g(S)$, and
	\item {\bf (approximately decreasing marginal gain)}
	\[g(S) - g(T) \leq \alpha \sum_{i \in T \backslash S}[g(S) - g(S \cup \{i\})].\]
\end{enumerate}
\end{definition}
The definition is slightly stronger than that in \cite{boutsidis2015greedy}. We
will show that sparse linear regression can be viewed as weakly supermodular
minimization in Definition~\ref{def: weak-sup-modu} once the design matrix
has bounded restricted condition number.

Now we outline the algorithm (see Algorithm~\ref{algorithm:sup-modu}). We divide
the rounds $1, 2,\ldots, T$ into mini-batches of size $B$ each (so there are
$T/B$ such batches). The $b$-th batch thus consists of the examples $(x_t, y_t)$
for $t \in \mc{T}_b := \{(b-1)B + 1, (b-1)B + 1, \ldots, bB\}$. Within the
$b$-th batch, our algorithm queries the same subset of features of size at most
$\sk$.

The algorithm consists of few key steps. First, one can show that under BBRCNP, as long as $B$ is large enough, the loss within batch $b$ defines a weakly supermodular
set function
$$
g_t(S) = \frac{1}{B} \inf\limits_{w \in \R^S}\sum\limits_{t \in \mc{T}_b} (y_t - \ip{x_t}{w})^2.
$$
Therefore, we can formulate the original online sparse regression problem into
online weakly supermodular minimization problem. For the latter problem, we
develop an online greedy algorithm along the lines of \citep{golovin-streeter}.
We employ $\wk = \O^*(k)$ budgeted experts algorithms \citep{AKTT}, denoted
BEXP, with budget parameter\footnote{We assume, for convenience, that $\sk$ is
divisible by $\wk$.} $\frac{\sk}{\wk}$. The precise characteristics of BEXP are
given in Theorem~\ref{theorem:bexp} (adapted from Theorem 2 in \citep{AKTT}).

\begin{theorem} \label{theorem:bexp}
For the problem of prediction from expert advice, let there be $d$ experts, and
let $k \in [d]$ be a budget parameter. In each prediction round $t$, the BEXP
algorithm chooses an expert $j_t$ and a set of experts $U_t$ containing $j_t$ of
size at most $k$, obtains as feedback the losses of all the experts in $U_t$,
suffers the loss of expert $j_t$, and guarantees an expected regret bound of
$2\sqrt{\frac{d\log(d)}{k}T}$ over $T$ prediction rounds.
\end{theorem}

At the beginning of each mini-batch $b$, the BEXP algorithms are run. Each BEXP
algorithm outputs a set of coordinates of size $\frac{\sk}{\wk}$ as well as a
special coordinate in that set. The union of all of these sets is then used as
the set of features to query throughout the subsequent mini-batch. Within the
mini-batch, the algorithm runs the standard Vovk-Azoury-Warmuth algorithm for
linear prediction with square loss {\em restricted} to set of special
coordinates output by all the BEXP algorithms.

At the end of the mini-batch, every BEXP algorithm is provided carefully
constructed losses for each coordinate that was output as feedback. These losses
ensure that the set of special coordinates chosen by the BEXP algorithms mimic
the greedy algorithm for weakly supermodular minimization.

\def\bexp{\textsf{BEXP}}
\def\vaw{\textsf{VAW}}
\begin{algorithm}[ht]
\begin{algorithmic}[1]
\REQUIRE Mini-batch size $B$, sparsity parameters $\sk$ and $\wk$

\STATE Set up $\wk$ budgeted prediction algorithms $\bexp^{(i)}$ for $i \in [\wk]$, each using the coordinates in $[d]$ as ``experts'' with a per-round budget of $\frac{\sk}{\wk}$.

\FOR{$b = 1,2,\ldots, T/B$}
\STATE For each $i \in [\wk]$, obtain a coordinate $j^{(i)}_b$ and subset of coordinates $U^{(i)}_b$ from $\bexp^{(i)}$ such that $j^{(i)}_b \in U^{(i)}_b$.
\STATE Define $V_b^{(0)} = \emptyset$ and for each $i \in [\wk]$ define $V^{(i)}_b = \{j^{(i')}_b\ |\ i' \leq i\}$.
\STATE Set up the Vovk-Azoury-Warmuth ($\vaw$) algorithm for predicting using the features in $V^{(\wk)}_b$.

\FOR{$t \in \mc{T}_b$}
\STATE Set $S_t = \bigcup_{i \in [\wk]} U^{(i)}_b$, obtain $x_t(S_t)$, and pass $x_t(V^{(\wk)}_b)$ to $\vaw$.
\STATE Set $w_t$ to be the weight vector output by $\vaw$.
\STATE Obtain the true label $y_t$ and pass it to $\vaw$.
\ENDFOR

\STATE Define the function
\begin{align}
\label{equation:gt}
g_b(S) = \frac{1}{B} \inf\limits_{w \in \R^S}\sum\limits_{t \in \mc{T}_b} (y_t - \ip{x_t}{w})^2.
\end{align}

\STATE For each $j \in U^{(i)}_b$, compute $g_b(V^{(i-1)}_b \cup \{j\})$ and pass it $\bexp^{(i)}$ as the loss for expert $j$.
\ENDFOR
\end{algorithmic}
\caption{Online Greedy Algorithm for POSLR}
\label{algorithm:sup-modu}
\end{algorithm}

\subsection{Main Result}
\label{section:adv.main}

In this section, we will show that Algorithm~\ref{algorithm:sup-modu} achieves
sublinear regret under BBRCNP.

\begin{theorem}
\label{theorem:adv}
Suppose the sequence of feature vectors satisfies BBRCNP with parameters
$(\kappa, \wk + k)$ for $\wk = \frac{1}{3}\kappa^2 k\log(T)$, and assume that $T$ is large enough so that $t_0 \leq (\frac{\sk T}{\kappa^2 d k})^{1/3}$. Then if Algorithm~\ref{algorithm:sup-modu} is run with parameters $B = (\frac{\sk T}{\kappa^2 d k})^{1/3}$ and
$\wk$ as specified above, its expected regret is at most
$\tilde{O}((\frac{\kappa^8 dk^4}{\sk})^{1/3}T^{2/3})$.
\end{theorem}

\begin{proof}
The proof relies on a number of lemmas whose proofs can be found in
the appendix. We begin with the connection between
sparse linear regression, weakly supermodular function and RIP, formally stated
in Lemma~\ref{lemma:RIP.modu}. This lemma is a direct consequence of Lemma 5 in
\citep{boutsidis2015greedy}.

\begin{lemma}
\label{lemma:RIP.modu}
Consider a sequence of examples $(x_t, y_t) \in \R^d \times \R$
for $t = 1, 2, \ldots, B$, and let $X$ be the design matrix for the sequence.
Consider the set function associated with least squares optimization:
\begin{align*}
g(S) &= \inf_{w \in \R^S}~ \frac{1}{B}\sum_{t=1}^B ( y_t - \ip{x_t}{w} )^2.
\end{align*}
Suppose the restricted condition number of $X$ for sparsity $k$ is bounded by $\kappa$.
Then $g(S)$ is $(k, \kappa^2)$-weakly supermodular.
\end{lemma}

Even though minimization of weakly supermodular functions is NP-hard, the greedy
algorithm provides a good approximation, as shown in the next lemma.

\begin{lemma}
\label{lemma:greedy-modular}
Consider a $(k, \alpha)$-weakly supermodular set function $g(\cdot)$. Let $j^* := \argmin_{j} g(\{j\})$. Then, for any subset $V$ of size at most $k$, we have
\begin{align*}
g(\{j^*\}) - g(V) \leq \left(1 - \tfrac{1}{\alpha |V|} \right) [g(\emptyset) - g(V)].
\end{align*}
\end{lemma}

The BEXP algorithms essentially implement the greedy algorithm in an online
fashion. Using the properties of the BEXP algorithm, we have the following
regret guarantee:
\begin{lemma}\label{lem:online-greedy}
	Suppose the sequence of feature vectors satisfies BBRCNP with parameters $(\epsilon, \wk + k)$. Then for any set $V$ of coordinates of size at most $k$, we have
	\begin{align*}
		& \E\left[\sum_{b = 1}^{T/B} g_b(V_b^{(\wk)}) - g_b(V)\right] \\
		&\leq \sum_{b = 1}^{T/B} \left(1 - \tfrac{1}{\kappa^2 |V|}\right)^\wk[g_b(\emptyset) - g_b(V)] + 2\kappa^2 k \sqrt{\tfrac{d\wk\log(d)T}{\sk B}}.
	\end{align*}
\end{lemma}

Finally, within every mini-batch, the VAW algorithm guarantees the following
regret bound, an immediate consequence of Theorem 11.8 in
\citet{cesa2006prediction}:
\begin{lemma}\label{lem:vaw}
Within every batch $b$, the VAW algorithm generates weight vectors $w_t$ for $t \in \mc{T}_b$ such that
\[ \sum_{t \in \mc{T}_b} (y_t - \ip{x_t}{w_t})^2 - Bg_b(V_b^{(\wk)}) \leq O(\wk\log(B)). \]
\end{lemma}

We can now prove Theorem~\ref{theorem:adv}. Combining the bounds of
lemma~\ref{lem:online-greedy} and \ref{lem:vaw}, we conclude that for any subset
of coordinates $V$ of size at most $k$, we have
\begin{align}
  &\E\left[\sum_{t=1}^T (y_t - \ip{x_t}{w_t})^2\right]\\
	&\leq \sum_{b=1}^{T/B} Bg_b(V) + B(1 - \tfrac{1}{\kappa^2|V|})^\wk[g_b(\emptyset) - g_b(V)]\\
	&+ O\left(\kappa^2k\sqrt{\tfrac{d\wk\log(d) BT}{\sk}} + \frac{T}{B}\wk\log(B)\right). \label{eq:adv-regret-bound}
\end{align}
Finally, note that
\[\sum_{b=1}^{T/B} Bg_b(V) \leq \inf_{w \in \R^V}\sum_{t=1}^T (y_t - \ip{x_t}{w})^2,\]
and
\[\sum_{b=1}^{T/B} B(1 - \tfrac{1}{\kappa^2|V|})^\wk[g_b(\emptyset) - g_b(V)] \leq T \cdot \exp(-\tfrac{\wk}{\kappa^2 k}),\]
because $g_b(\emptyset) \leq 1$. Using these bounds in \eqref{eq:adv-regret-bound}, and plugging in the specified values of $B$ and $\wk$, we get the stated regret bound.
\end{proof}

\section{Conclusions and Future Work}

In this paper, we gave computationally efficient algorithms for the online sparse linear regression problem under the assumption that the design matrices of the feature vectors satisfy RIP-type properties. Since the problem is hard without any assumptions, our work is the first one to show that assumptions that are similar to the ones used to sparse recovery in the batch setting yield tractability in the online setting as well.

Several open questions remain in this line of work and will be the basis for future work. Is it possible to improve the regret bound in the agnostic setting? Can we give matching lower bounds on the regret in various settings? Is it possible to relax the RIP assumption on the design matrices and still have efficient algorithms? Some obvious weakenings of the RIP assumption we have made don't yield tractability. For example, simply assuming that the final matrix $X_T$ satisfies RIP rather than every intermediate matrix $X_t$ for large enough $t$ is not sufficient; a simple tweak to the lower bound construction of \citet{foster2016online} shows this. This tweak consists of simply padding the construction with enough dummy examples which are well-conditioned enough to overcome the ill-conditioning of the original construction so that RIP is satisfied by $X_T$. We note however that in the realizable setting, our analysis can be easily adapted to work under weaker conditions such as irrepresentability \citep{zhao2006model,javanmard2013model}.

\bibliography{ref}
\bibliographystyle{plainnat}

\newpage
\appendix

\section{Proofs for Realizable Setting}
\label{section:realizable-proofs}

\begin{proof}[Proof of Lemma~\ref{lemma:dantzig}]
Let $\Delta:= \emp{w} - w^*$ be the difference between the true answer and
solution to the optimization problem. Let $S$ to be the support of $w^*$
and let $S^c = [d] \setminus S$ be the complements of $S$.
Consider the permutation $i_1,\ldots,i_{d-k}$ of $S^c$ for which $|\Delta(i_j)|
\ge |\Delta(i_{j+1})|$ for all $j$. That is, the permutation dictated by the
magnitude of the entries of $\Delta$ outside of $S$. We split $S^c$ into subsets
of size $k$ according to this permutation: Define $S_j$, for $j\ge 1$ as
$\{i_{(j-1)k+1},\dots,i_{jk}\}$. For convenience we also denote by $S_{01}$ the
set $S \cup S_1$.

Now, consider the matrix $X_{S_{01}} \in \R^{t \times |S_{01}|}$ whose columns
are those of $X$ with indices $S_{01}$. The Restricted Isometry Property of $X$
dictates that for any vector $c \in \R^{S_{01}}$,
$$
(1-\epsilon)\norm{c}_2 \le \frac{1}{\sqrt{n}} \norm{X_{S_{01}}c}_2 \le (1+\epsilon) \norm{c}_2.
$$
Let $V \subseteq \R^t$ be the subspace of dimension $|S_{01}|$ that is the image
of the linear operator $X_{S_{01}}$, and let $P_V \in \R^{t \times t}$ be the
projection matrix onto that subspace. We have, for any vector $z \in \R^t$ that
$$
(1-\epsilon) \norm{P_V z} \le \frac{1}{\sqrt{n}} \norm{X^T_{S_{01}} z} \le (1+\epsilon) \norm{P_V z}
$$
We apply this to $z = X\Delta$ and conclude that
\begin{equation}\label{equation:pvxd_upper}
\norm{P_V X \Delta} \le \frac{1}{\sqrt{t}(1-\epsilon)} \norm{X_{S_{01}}^T X \Delta}
\end{equation}
We continue to lower bound the quantity of $\norm{P_V X \Delta}$. We decompose
$P_V X \Delta$ as
\begin{equation} \label{equation:pvxd}
P_V X \Delta  = P_V X \Delta(S_{01}) + \sum_{j \ge 2} P_V X \Delta(S_j)
\end{equation}
Now, according to the definition of $V$ we that there exist vectors $\{c_j\}_{j
\ge 2}$ in $\R^{|S_{01}|}$ for which
$$
P_V X \Delta(S_j) = X_{S_{01}}c_j
$$
We now invoke Lemma 1.1 from \cite{candes2005decoding} stating that for any
$S',S''$ with $|S'|+|S''| \le 3k$ it holds
that
$$
\forall c,c' \ \ \ \frac{1}{n} \ip{X_{S'}c}{X_{S''}c'} \le (2\epsilon -\epsilon^2) \norm{c}_2 \norm{c'}_2
$$
We apply this for $S_{01}, S_j$, $j \ge 2$ and conclude that
$$
\norm{P_V X \Delta(S_j)}_2^2
= \ip{P_V X \Delta(S_j)}{X \Delta(S_j)}
\le 2\epsilon t \norm{c_j}_2 \cdot \norm{\Delta(S_j)}
\le \frac{2\epsilon \sqrt{t}}{1-\epsilon} \norm{P_V X \Delta(S_j)}_2 \cdot \norm{\Delta(S_j)}_2.
$$
Dividing through by $\norm{P_V X \Delta(S_j)}_2$, we get
\begin{equation}
\label{equation:pvxsj}
\norm{P_V X \Delta(S_j)} \le \frac{2\epsilon \sqrt{t}}{1-\epsilon} \norm{\Delta(S_j)}.
\end{equation}
Let us now bound the sum $\norm{\Delta(S_j)}$. By the definition of $S_j$ we
know that any element $i \in S_j$ has the property $\Delta(i) \le (1/k)
\norm{\Delta(S_{j-1})}_1$. Hence
$$
\sum_{j \ge 2} \norm{\Delta(S_j)} \le (1/\sqrt{k}) \sum_{j \ge 1} \norm{\Delta(S_j)}_1 = (1/\sqrt{k}) \norm{\Delta(S^c)}_1
$$
We now combine this inequality with
Equations~\eqref{equation:pvxd_upper},~\eqref{equation:pvxd} and~\eqref{equation:pvxsj}
\begin{align*}
\frac{1}{t} \norm{X_{S_{01}}^T X \Delta}
& \ge \frac{1-\epsilon}{\sqrt{t}} \norm{P_V X \Delta} \\
& \ge \frac{1-\epsilon}{\sqrt{t}} \norm{P_V X \Delta(S_{01})} - \frac{1-\epsilon}{\sqrt{n}} \sum_{j \ge 2} \norm{P_V X \Delta(S_j)} \\
& \ge \frac{1-\epsilon}{\sqrt{t}} \norm{X \Delta(S_{01})} - 2\epsilon \sum_{j \ge 2} \norm{\Delta(S_j)} \\
& \ge \frac{1-\epsilon}{\sqrt{t}} \norm{X \Delta(S_{01})} - \frac{2\epsilon}{\sqrt{k}} \norm{\Delta(S^c)}_1
\end{align*}
The third inequality holds since $X \Delta(S_{01}) \in V$ hence $P_V X
\Delta(S_{01)} = X \Delta(S_{01})$. We continue to bound the expression by
claiming that $\norm{\Delta(S)}_1 \ge \norm{\Delta(S^c)}_1$. This holds since in
$S^c$, $\emp{w}_{S^c}=\Delta(S^c)$ hence
$$
\norm{w^*}_1 = \norm{\emp{w} - \Delta(S^c) - \Delta(S)}_1
\le \norm{\emp{w}}_1 + \left( \norm{\Delta(S)}_1 - \norm{\Delta(S^c)}_1 \right)
$$
Now, the optimality of $\emp{w}$ implies $\norm{\emp{w}}_1 \le \norm{w^*}_1$, hence
indeed $\norm{\Delta(S)}_1 \ge \norm{\Delta(S^c)}_1$.
$$
\norm{\Delta(S^c)}_1
\le \norm{\Delta(S)}_1
\le \sqrt{k} \norm{\Delta(S)}_2
\le \norm{\Delta(S_{01})}_2
\le \frac{\sqrt{k}}{(1-\epsilon)\sqrt{t}} \norm{X \Delta(S_{01})}
$$
We continue the chain of inequalities
\begin{align*}
\frac{1}{t} \norm{X_{S_{01}}^T X \Delta}
& \ge \frac{1-\epsilon}{\sqrt{n}} \norm{X \Delta(S_{01})} - \frac{2\epsilon}{\sqrt{k}} \norm{\Delta(S^c)}_1 \\
& \ge \norm{X \Delta(S_{01})} \left( \frac{1-\epsilon}{\sqrt{n}} - \frac{2\epsilon}{\sqrt{k}} \cdot  \frac{\sqrt{k}}{(1-\epsilon)\sqrt{n}} \right) \\
& = \frac{(1-\epsilon)^2 - 2\epsilon}{(1-\epsilon)\sqrt{t}} \norm{X \Delta(S_{01})}
\end{align*}
Rearranging we conclude that
\begin{align*}
\norm{\Delta(S_{01})}
& \le \frac{1}{(1-\epsilon)\sqrt{t}}  \norm{X \Delta(S_{01})} & \text{(RIP of $X$)} \\
& \le \frac{1}{((1-\epsilon)^2 - 2\epsilon)t} \norm{X_{S_{01}}^T X \Delta} \\
& \le \frac{\sqrt{2k}}{(1 - 4\epsilon)t} \norm{X^T X \Delta}_\infty & \text{(since for any $z \in \R^{2k}$, $\norm{z}_2 \le \sqrt{2k} \norm{z}_\infty$)} \\
& \le C \sqrt{ \frac{d k \log (d/\delta)}{ t \sk} } \left(\sigma + \frac{d}{\sk} \norm{w^*}_1 \right) & \text{(Lemma~\ref{lemma:localization} and $\epsilon < 1/5$)} \\
\end{align*}
for some constant $C$. We continue our bound on $\norm{\Delta}$ by showing that
$\norm{\Delta(S_{01}^c)} \le \norm{\Delta(S_{01})}$
\begin{align*}
\norm{\Delta(S_{01}^c)}_2^2 \stackrel{(i)}{\le} \norm{\Delta(S^c)}_1^2 \cdot \sum_{j \ge k+1} \frac{1}{j^2}
\le \frac{1}{k} \norm{\Delta(S^c)}_1^2 \le \frac{1}{k} \norm{\Delta(S)}_1^2
\le \norm{\Delta(S)}_2^2.
\end{align*}
Inequality $(i)$ holds due to the following: Let $\alpha_i$ be the absolute
value of the $i$'th largest (in absolute value) element of $\Delta(S^c)$. It
obviously holds that $\alpha_i \le \norm{\Delta(S^c)}_1/i$. Now, according to the
definition of $S_{01}$ we have that $\norm{\Delta(S_{01}^c)}_2^2 = \sum_{j \ge
k+1} \alpha_i^2$ and the inequality follows.
Hence,
$$
\norm{\Delta(S_{01}^c)}_2
\le \norm{\Delta(S)}_2
\le \norm{\Delta(S_{01})}_2.
$$
We conclude that
$$
\norm{\Delta}_2
\le \sqrt{2} \norm{\Delta(S_{01})}_2
\le C \sqrt{ \frac{d k \log (d/\delta)}{ t \sk} } \left(\sigma + \frac{d}{\sk} \norm{w^*}_1 \right)
$$
for some universal constant $C > 0$. Since $\norm{\Delta(S)}_1 \ge \norm{\Delta(S^c)}_1$ and $|S| \leq k$ we get that
$$ \norm{\Delta}_1 \leq 2\norm{\Delta(S)}_1 \leq 2\sqrt{k}\norm{\Delta(S)}_2 \leq 2\sqrt{k}\norm{\Delta}_2 $$
and the claim follows.
\end{proof}

\begin{proof}[Proof of Lemma~\ref{lemma:sparse}]
Let $S$ be the support of $w^*$. We can decompose the square of the left hand
side as
$$
\norm{\emp{w}(\widetilde{S})  - w^*}_2^2
= \sum_{i \in S \cap {\widetilde{S}}} (\emp{w}(i) - w^*(i))^2 + \sum_{i \in \widetilde{S} \setminus S} (\emp{w}(i))^2 + \sum_{i \in S \setminus \widetilde{S}} (w^*(i))^2.
$$
We upper bound the last sum on the right hand side as
\begin{align*}
\sum_{i \in S \setminus \widetilde{S}} (w^*(i))^2
& = \sum_{i \in S \setminus \widetilde{S}} \left[(\emp{w}(i) - w^*(i)) + (\emp{w}(i)) \right]^2  \\
& \le 2 \sum_{i \in S \setminus \widetilde{S}}  (\emp{w}(i) - w^*(i))^2 + (\emp{w}(i))^2 \\
& \le 2 \sum_{i \in S \setminus \widetilde{S}}  (\emp{w}(i) - w^*(i))^2 + 2 \sum_{i \in \widetilde{S} \setminus S}  (\emp{w}(i))^2 \; ,
\end{align*}
where first inequality follows from the elementary inequality $(a+b)^2 \le 2a^2 +
2b^2$ and the second inequality is due to the fact that $\widetilde{S}$
contains top $k$ entries of $\emp{w}$ in absolute value and $|S \setminus
\widetilde{S}| = |\widetilde{S} \setminus S|$.
Hence,
\begin{align*}
\norm{\emp{w}(\widetilde{S})  - w^*}_2^2
& = \sum_{i \in S \cap {\widetilde{S}}} (\emp{w}(i) - w^*(i))^2 + \sum_{i \in \widetilde{S} \setminus S} (\emp{w}(i))^2 + \sum_{i \in S \setminus \widetilde{S}} (w^*(i))^2 \\
& \le \sum_{i \in S \cap {\widetilde{S}}} (\emp{w}(i) - w^*(i))^2 + 2 \sum_{i \in S \setminus \widetilde{S}}  (\emp{w}(i) - w^*(i))^2 + 3 \sum_{i \in \widetilde{S} \setminus S} (\emp{w}(i))^2 \\
& \le 2 \sum_{i \in S \cap {\widetilde{S}}} (\emp{w}(i) - w^*(i))^2 + 2 \sum_{i \in S \setminus \widetilde{S}}  (\emp{w}(i) - w^*(i))^2 + 3 \sum_{i \in \widetilde{S} \setminus S} (\emp{w}(i))^2 \\
& = 2 \sum_{i \in S} (\emp{w}(i) - w^*(i))^2 + 3 \sum_{i \in \widetilde{S} \setminus S} (\emp{w}(i))^2 \\
& \le 3 \sum_{i=1}^d (\emp{w}(i) - w^*(i))^2 \\
& = 3 \norm{\emp{w} - w^*}_2^2.
\end{align*}
Taking square root finishes the proof.
\end{proof}

\begin{lemma}
\label{lemma:localization}
There exists a universal constant $C > 0$ such that, with probability at least $1-\delta$, the convex
program~\eqref{equation:conv.prog} is feasible and its optimal solution $\emp{w}$ satisfies
$$
\norm{\frac{1}{t} X_t^T X_t (\emp{w} - w^*)}_{\infty} \le
C \sqrt{ \frac{d \log (d/\delta)}{t \sk} } \left(\sigma + \frac{d}{\sk} \norm{w^*}_1 \right).
$$
\end{lemma}

We note that the above lemma is beyond simple triangle inequality on the
feasibility constraints, as the left hand side depends on actual design matrix
$X_t$ which we do not observe, instead of $\emp{X}_t$.

\begin{proof}
To simplify notation, we drop subscript $t$. Namely, let $X = X_t$, $\emp{X} =
X_t$ and $\emp{D} = \emp{D}_t$, and also let $\eta = (\eta_1, \eta_2, \dots,
\eta_t)$ be the vector of noise variables.

First, we show that $w^*$ satisfies the constraint of~\eqref{equation:conv.prog} with probability
at least $1-\delta$. We upper bound
\begin{align*}
\norm{ \frac{1}{t} \emp{X}^T (Y - \emp{X} w^*)  + \frac{1}{t} \emp{D} w^*  }_{\infty}
& = \norm{ \left[ \frac{1}{t}  \emp{X}^T (X - \emp{X}) + \frac{1}{t} \emp{D} \right] w^* + \frac{1}{t} \emp{X}^T \eta }_{\infty} \\
& \le \norm{ \left[\frac{1}{t} \emp{X}^T (X - \emp{X})  + \frac{1}{t} \emp{D}  \right]w^* }_{\infty} + \frac{1}{t} \norm{ \emp{X}^T \eta }_{\infty}
\end{align*}
We first bound the left summand. By Lemma~\ref{lemma:concentration-inequality}, we have
\begin{align*}
\norm{ \left[ \frac{1}{t} \emp{X}^T (X - \emp{X})  +\frac{1}{t} \emp{D}  \right] w^* }_{\infty}
& \le \norm{w^*}_{1} \cdot \norm{   \frac{1}{t}  \emp{X}^T (X - \emp{X}) + \frac{1}{t} \emp{D} }_{\infty}  \\
& \le \norm{w^*}_{1} \left( \norm{ \frac{1}{t} X^T(\emp{X} -X ) }_{\infty} + \norm{ \frac{1}{t} (\emp{X} - X)^T (\emp{X} - X) - \frac{1}{t}\emp{D}  }_{\infty} \right) \\
& \le \norm{w^*}_{1} C \cdot \sqrt{\frac{d^3 \log (d/\delta)}{t \sk^3}}.
\end{align*}
For the right summand, since $\eta$ is vector of i.i.d Gaussians with variance
$\sigma^2$, with probability at least $1-\delta$,
$$
\frac{1}{t} \norm{ \emp{X}^T \eta }_{\infty} \le   C \frac{\sigma}{t} \sqrt{\log (d/\delta)}  \cdot \max_{i \in [d]} \norm{\emp{X}_{(i)}}_2
$$
where $\emp{X}_{(1)}, \emp{X}_{(2)}, \dots, \emp{X}_{(d)}$ are the columns of $\emp{X}$.
Since the absolute value of the entries of $\emp{X}$ is at most $d/\sk$,
we have $\norm{\emp{X}_{(i)}}_2 \le \sqrt{td/\sk}$ and thus
$$
\frac{1}{t} \norm{ \emp{X}^T \eta }_{\infty} \le   C \sigma \sqrt{\frac{d\log (d/\delta)}{t\sk}}.
$$
Combining the inequalities so far provides
$$
\norm{ \frac{1}{t} \emp{X}^T ( Y - \emp X w^* )  + \frac{1}{t} \emp{D} w^* }_{\infty} \le C \sqrt{ \frac{d \log (d/\delta)}{t \sk} } \left(\sigma + \frac{d}{\sk} \norm{w^*}_1 \right)
$$
and hence conclude the constraint of the optimization problem~\eqref{equation:conv.prog}
is satisfied (at least) by $w^*$ and thus the optimization problem is feasible.

Now consider the vector $\Delta:=\emp{w} - w^*$, we have
\begin{align*}
\norm{ \frac{1}{t} X^T X \Delta  }_{\infty}
& \le  \norm{ \frac{1}{t} (\emp{X}^T \emp{X}  - \emp{D}) \Delta }_\infty + \norm{ \frac{1}{t} (\emp{X}^T \emp{X}  - \emp{D} - X^TX) \Delta }_\infty \\
& \le \norm{ \frac{1}{t} (\emp{X}^T \emp{X}  - \emp{D}) \Delta }_\infty + \norm{ \frac{1}{t} (\emp{X} - X)^T X \Delta }_\infty \\
& \qquad + \norm{ \frac{1}{t} X^T (\emp{X} - X) \Delta }_\infty + \norm{ \left(\frac{1}{t}(\emp{X} - X)^T (\emp{X} - X) - \frac{1}{t} \emp{D} \right)   \Delta  }_\infty.
\end{align*}
According to Lemma~\ref{lemma:concentration-inequality} we have
$$
\norm{\frac{1}{t} X^T (\emp{X} - X) \Delta}_{\infty}
\le \norm{\frac{1}{t} X^T (\emp{X} - X) }_{\infty} \norm{ \Delta }_1
\le C \sqrt{\frac{d \log (d/\delta)}{t \sk}} ( \norm{ w^*}_1 + \norm{ \emp{w} }_1 ) \le
2C \sqrt{\frac{d \log (d/\delta)}{t \sk}} \cdot \norm{w^*}_1
$$
where the last inequality is by the optimality of $\emp{w}$. The same argument
provides an identical bound for $\norm{\frac{1}{t} (\emp{X} - X)^T X
\Delta}_{\infty}$. The last summand can also be bounded by using
Lemma~\ref{lemma:concentration-inequality} and the optimality of $\emp{w}$.
$$
\norm{\left(\frac{1}{t}(\emp{X} - X)^T (\emp{X} - X) - \frac{1}{t}\emp{D} \right) \Delta}_{\infty}
\le 2C \sqrt{\frac{d^3 \log (d/\delta)}{t \sk^3}} \cdot \norm{w^*}_1
$$
Finally, according to the feasibility of $\emp{w}$ and $w^*$ we may bound the first summand
$$
\norm{\left( \frac{1}{t} \emp{X}^T \emp{X} - \frac{1}{t} \emp{D} \right) \Delta}_{\infty}
\le 2 C \sqrt{ \frac{d \log (d/\delta)}{t \sk} } \left(\sigma + \frac{d}{\sk} \norm{w^*}_1 \right),
$$
and reach the final bound.
\end{proof}

\begin{lemma}
\label{lemma:concentration-inequality}
For any $t \ge t_0$, with probability at least $1-\delta$, the following two inequalities hold
\begin{align*}
\norm{\frac{1}{t} (\emp{X}_t - X_t)^T (\emp{X}_t - X_t) - \frac{1}{t} \emp{D}_t}_{\infty} \le  C \sqrt{\frac{d^3 \log (d/\delta)}{t {\sk} ^3}} \; , \\
\norm{\frac{1}{t} X_t^T (\emp{X}_t - X_t)}_{\infty} \le C \sqrt{\frac{d \log (d/\delta)}{t \sk}} \; ,
\end{align*}
where $\norm{\cdot}_\infty$ denotes the maximum of the absolute values of the entries of a matrix.
\end{lemma}

\begin{proof}
Throughout we use that $|x_s(i)| \le 1$ for all $i \in [d]$ and all $s \in [t]$, and (2)
$(\emp{x}_s(i) - x_s(i))^2  - \frac{1}{t}D_{ii}$ is unbiased with absolute
value of at most $(d/\sk)^2$ and variance of at most $(d/\sk)^3$. For the first
term, let's bound
$$
\left[ \frac{1}{t} (\emp{X} - X)^T (\emp{X} - X) - \frac{1}{t} \emp{D}  \right]_{ij}
= \frac{1}{t} \sum_{s=1}^t (\emp{x}_s(i) - x_s(i))(\emp{x}_s(j) - x_s(j)) - \frac{1}{t} \emp{D}_{ij}
$$
For $i = j$, we have
$$
\E\left[ \left( (\emp{x}_s(i) - x_s(i))^2  - \frac{1}{t} D_{ii}\right)^2 \right] \le \E\left[ (\emp{x}_s(i) - x_s(i))^4 \right] \le (d/\sk)^3
$$
$$
(\emp{x}_s(i) - x_s(i))^2 - \frac{1}{t} D_{ii} \le (d/\sk)^2, \ \ \ \E\left[ (\emp{x}_s(i) - x_s(i))^2 - \frac{1}{t} D_{ii} \right] = 0
$$
Hence, by Bernstein's inequality, for any $v > 0$,
$$
\Pr\left[  \left| \frac{1}{t} \sum_{s=1}^t (\emp{x}_s(i) - x_s(i))^2 - \frac{1}{t} D_{ii} \right| > v   \right]
\le 2\exp\left( -\frac{v^2 t}{(d/\sk)^3 + (d/\sk)^2 v/3} \right).
$$
It follows that for any $\delta>0$, with probability at least $1-\delta$ it holds for all $i \in [d]$ that,
$$
\left| \frac{1}{t} \sum_{s=1}^t (\emp{x}_s(i) - x_s(i))^2 - \frac{1}{t} D_{ii} \right|
\le \O \left(\frac{\log(d/\delta) d^2}{t\sk^2} + \sqrt{\frac{\log(d/\delta)d^3}{t\sk^3}} \right).
$$
Similarly we have $\frac{1}{t}(\emp{D}_{ii} - D_{ii}) \le \O \left(\frac{\log(d/\delta) d^2}{t\sk^2} + \sqrt{\frac{\log(d/\delta)d^3}{t\sk^3}} \right).$

For $i \neq j$ we use an analogous argument, only now the variance term in
Bernstein's inequality is $(d/\sk)^2$ rather than $(d/\sk)^3$, hence only reach
a tighter bound.

For the second term, we again bound via Bernstein's inequality as
$$
\left[  \frac{1}{t} X^T (\emp{X} - X) \right]_{ij}
= \frac{1}{t} \sum_{s=1}^t  x_s(i)(\emp{x}_s(j) - x_s(j))
\le \O \left( \sqrt{\frac{d \log (d/\delta)}{t \sk}} + \frac{d \log (d/\delta)}{t\sk} \right)
$$
The claim now follows by noticing that for large enough $t$, the dominating
terms are those that scale as $1/\sqrt{t}$.
\end{proof}

\begin{proof}[Proof of Theorem~\ref{theorem:realizable}]
By Lemma~\ref{lemma:dantzig},
$$
\norm{w_{t+1} - w^*}_2 \le  \O\left( \sqrt{\frac{d}{\sk} \frac{k \log (d/\delta)}{t}} (\sigma + \frac{d}{\sk} \norm{w^*}_1) \right).
$$
We have
\begin{align*}
\Regret_T(w^*) - \Regret_{t_0}(w^*)
& = \sum_{t=t_0+1}^T (y_t - \ip{x_t}{w_t})^2  - (y_t - \ip{x_t}{w^*})^2 \\
& = \sum_{t=t_0+1}^T (\ip{x_t}{w^* - w_t} + \eta_t)^2  - \eta_t^2 \\
& = \sum_{t=t_0+1}^T \left(\ip{x_t}{w^* - w_t} + 2 \eta_t \right) \ip{x_t}{w^* - w_t} \\
& = \sum_{t=t_0+1}^T 2\eta_t \ip{x_t}{w^* - w_t} + \left(\ip{x_t}{w^* - w_t} \right)^2 \; ,
\end{align*}
where we used that $y_t = \ip{x_t}{w_t} + \eta_t$.
To bound the regret we require the upper bound, that occurs with probability of at least $1-\delta$,
$$
\forall t \ge t_0 \qquad
\left| \ip{x_t}{w^* - w_t} \right|
\stackrel{(i)}{\le} \norm{x_t}_\infty \sqrt{\norm{w_t - w^*}_0} \cdot \norm{w_t - w^*}_2
\stackrel{(ii)}{\le} \O \left( k \cdot  \sqrt{\frac{d}{\sk} \frac{\log (\log(T)d/\delta)}{t}} \left(\sigma + \frac{d}{\sk} \right) \right).
$$
Inequality $(i)$ holds since $\ip{a}{b} \le \norm{a(S)}_2 \cdot \norm{b}_2$ with $S$ being
the support of $b$ and $\norm{a(S)}_2 \le \norm{a}_\infty \sqrt{|S|}$. Inequality $(ii)$
follows from Lemma~\ref{lemma:dantzig} and Lemma~\ref{lemma:sparse}, and a
union bound over the $\lceil \log(T) \rceil$ many times the vector $w_t$ is updated.
Now, for the left summand of the regret bound we have by Martingale concentration inequality
that w.p. $1-\delta$
\begin{align*}
\sum_{t=t_0+1}^T 2\eta_t \ip{x_t}{w_t-w^*}
& \le \O \left(  \sigma \sqrt{ \log(1/\delta) \sum_{t=t_0+1}^T \ip{x_t}{w_t - w^*}^2  } \right) \\
& =  \O \left(  \sigma \sqrt{ \log(1/\delta) \log(T) k^2 \cdot  \frac{d \log (d\log(T)/\delta)}{\sk}  \left(\sigma + \frac{d}{\sk} \right)^2} \right).
\end{align*}
The right summand is bounded as
$$
\sum_{t=t_0+1}^T  \ip{x_t}{w^* - w_t}^2 =
\O \left( k^2 \cdot  \frac{d \log (d\log(T)/\delta)}{\sk}  \left(\sigma + \frac{d}{\sk} \right)^2 \cdot \log(T)  \right).
$$
Clearly, the right summand dominates the left one.

It remains to bound the regret in first $t_0$ rounds. Since $w_t = 0$ for $t \le t_0$, we have
$$
\Regret_{t_0}(w^*)
= \sum_{t=1}^{t_0} 2\eta_t \ip{x_t}{w^*} + \left(\ip{x_t}{w^*} \right)^2
\le \O \left( \sigma \sqrt{t_0  \log(1/\delta)} + t_0 \right).
$$
Here, we used that  $|\ip{x_t}{w^*}| \le 1$ since $\norm{x_t}_\infty \le 1$ and $\norm{w^*}_1 \le 1$.
We also used that $\eta_t \ip{x_t}{w^*} \sim N(0, \sigma^2 \ip{x_t}{w^*}^2)$
and $\eta_1 \ip{x_1}{w^*}, \eta_2 \ip{x_2}{w^*}, \dots, \eta_{t_0} \ip{x_{t_0}}{w^*}$
are independent. Thus their sum is a Gaussian with variance at most $\sigma^2 t_0$.

Collecting all the terms along with Lemma~\ref{lem:wstart_reg}, bounding the difference $\Regret_T - \Regret_T(w^*)$, gives
\begin{equation}
\label{equation:regret-realizable}
\Regret_T \le \left(t_0 + \sqrt{t_0 \log(1/\delta)} + k^2 \cdot  \frac{d \log (d\log(T)/\delta)}{\sk}  \left(\sigma + \frac{d}{\sk} \right)^2 \cdot \log(T)  \right)
\end{equation}
\end{proof}

\begin{lemma} \label{lem:wstart_reg}
In the realizable case, w.p.\ at least $1-\delta$ we have for any sequence of $w_t$ that $\Regret_T - \Regret_T(w^*) = O(\sigma^2 k \log(d/\delta))$.
\end{lemma}
\begin{proof}
It is an easy exercise to show that $\Regret_T - \Regret_T(w^*)$ is equal to the regret on an algorithm that always plays $w^*$. We thus continue to bound the regret of $w^*$.

Let $\Delta \in \R^d$ be the difference between $w^*$ and $\tilde{w}$, the empirical optimal solution for the sparse regression problem. The loss associated with $w^*$ is clearly $\|\eta\|^2$, where $\eta$ is the noise term $y = X w^* +\eta $.
The loss associated with $\tilde{w}$ is
$$ \| X(w^* + \Delta) - Xw^* - \eta \|^2 = \|\eta - X\Delta \|^2 = \|\eta - X_{\tilde{S}} \Delta \|^2 $$
where $\tilde{S}$ is the support of $\Delta$, having a cardinality of at most $2k$. The closed form solution for the least-squares problem dictates that
$$ \|\eta - X_{\tilde{S}} \Delta \|^2 \geq \| \eta - X_{\tilde{S}} X_{\tilde{S}}^\dagger \eta \|^2 = \|\eta\|^2 - \| X_{\tilde{S}} X_{\tilde{S}}^\dagger \eta \|^2 \ .$$
Here, $A^\dagger$ is the pseudo inverse of a matrix $A$ and $X_S$ is the matrix obtained from the columns of $X$ whose indices are in $S$. It follows that the regret of $w^*$ is bounded by
$$  \| X_{\tilde{S}} X_{\tilde{S}}^\dagger \eta \|^2 $$
for some subset $\tilde{S}$ of size at most $2k$. To bound this quantity we use a high probability bound for $ \| X_{S} X_{S}^\dagger \eta \|^2 $ for a fixed set $S$, and take a union bound over all possible sets of cardinality $2k$. For a fixed set $S$ we have that $ \| X_{S} X_{S}^\dagger \eta \|^2 / \sigma^2 $ is distributed according to the $\chi_{2k}^2$ distribution. The tail bounds of this distribution suggest that
$$ \Pr\left[  \| X_{S} X_{S}^\dagger \eta \|^2  > 2k \sigma^2 + 2\sigma^2 \sqrt{2k x} + 2 \sigma^2 x  \right] \leq \exp(-x)$$
meaning that with probability at least $1- \delta/d^{2k}$ we have
$$ \| X_{S} X_{S}^\dagger \eta \|^2 < 2k \sigma^2 + 2\sigma^2 \sqrt{2k \cdot 2k\cdot \log(d/\delta)} + 2 \sigma^2 \cdot 2k \cdot \log(d/\delta) = O(\sigma^2 k \log(d/\delta))  $$
Taking a union bound over all possible subsets of size $\leq 2k$ we get that w.p.\ at least $1-\delta$ the regret of $w^*$ is at most $O(\sigma^2 k \log(d/\delta))$.
\end{proof}

\section{Proofs for Agnostic Setting}
\label{section:agnostic-proofs}

\def\barX{\bar{X}}

We begin with an auxiliary lemma for  Lemma~\ref{lemma:RIP.modu}, informally proving that for any matrix $\barX$ with BBRCNP (Definition~\ref{definition:block-BRCNP}) and vector $y$, the set function
\[ g(S) = \inf_{w \in \R^S} \|y - \barX w\|^2\]
is weakly supermodular.
Its proof can be found in \citep{boutsidis2015greedy}, yet for completeness we provide it here as well.

\begin{lemma} \label{lem:regression_sub} [Lemma 5 in \citep{boutsidis2015greedy}]
Let $\barX$ be a matrix whose columns have 2-norm at most $1$ and $y$ be a vector with $\|y\|_\infty \leq 1$ of dimension matching the number of rows in $X$. the set function
\[ g(S) = \inf_{w \in \R^S} \|y - X w\|^2\]
is $\alpha$-weakly supermodular for sparsity $k$ for $\alpha = \max_{S: |S| \leq k} 1/\sigma_{\min}(X_S)^2$, where $X_S$ is the submatrix of $X$ obtained by choosing the columns indexed by $S$, and $\sigma_{\min}(A)$ is the smallest singular value of $A$.
\end{lemma}
\begin{proof}
Firstly, the well known closed form solution for the least-squares problem informs us that
\begin{align*}
	g(S) & = \inf_{w \in \R^S}~ \| y - X w \|^2, \\
	&= y^T [ I - (X_{S}^T)^\dagger X_{S}^T ] y.
\end{align*}
We use the notation $A^\dagger$ for the pseudoinverse of a matrix $A$. That is, if the singular value decomposition of $A$ is $A=\sum_i \sigma_i u_i v_i^T$ with $\sigma_i>0$ then $A^\dagger = \sum_i \sigma_i^{-1} v_i u_i^T$.

\newcommand{\Zts}{Z_{T \setminus S}}

Let us first estimate $g(S)-g(T)$, for sets $S \subset T$. For brevity, define $H_S$ as the projection matrix $ X_{S} X_S^\dagger$ projecting onto the column space of $X_S$. Denote by $\Zts$ the matrix whose columns are those of $X_{T \setminus S}$ projected away from the span of $X_S$, and normalized. Namely, writing $x_i$ as the $i$'th column of $X$, $\zeta_i = \|(I-H_S) x_i\|$, $z_i = (I-H_S) x_i/\zeta_i$, and $\Zts$'s columns are $\{z_i\}_{i \in T \setminus S}$. Notice that the columns of $\Zts$ and $X_S$ are orthogonal, hence according to the Pythagorean theorem it holds that
$$
g(S) = \|y\|^2 - \|H_S y\|^2, \ \ \ g(T) = \|y\|^2 - \|H_S y\|^2 - \|\Zts \Zts^\dagger y\|^2
$$
meaning that $g(S)-g(T) = \|\Zts \Zts^\dagger y\|^2$. In particular, this implies that for any $j \notin S$ it holds that $g(S)-g(S \cup \{j\}) = (z_j^T y)^2$, since $z_j$ is a unit vector. Let us now decompose $g(S)-g(T)$.
\[
	g(S) - g(T) = \| \Zts \Zts^\dagger y   \|^2 =  \| (\Zts^T)^\dagger \Zts^T y   \|^2 \leq \| (\Zts^T)^\dagger \|^2 \cdot \| \Zts^T y \|^2
\]
The norm used in the last inequality is the matrix operator norm. We now bound both factors of the product on the RHS separately. For the first factor, we claim that $\| (\Zts^T)^\dagger \| = \| \Zts^\dagger \| \leq \|X_{T}^\dagger \|$. To see this, consider a vector $w \in \R^{|T \setminus S|}$, for convenience denote its entries by $\{w(i)\}_{i \in T \setminus S}$, and write $z_i = (x_i - \sum_{j \in S} \alpha_{ij} x_j)/\zeta_i$. We have
$$ \Zts w = \sum_{i \in T \setminus S} z_i w(i) = \sum_{i \in T \setminus S} x_i w(i) / \zeta_i - \sum_{j \in S} x_j \sum_{i \in T \setminus S} w(i) \alpha_{ij}/\zeta_i = X_{T} w'  $$
for the vector $w' \in \R^{|T|}$ defined as $w'(i) = w(i)/\zeta_i$ for $i \in T\setminus S$ and $w'(j) = - \sum_{i \in T \setminus S} w(i)\alpha_{ij}/\zeta_i$ for $j \in S$. Since $\zeta_i \leq \|x_i\| \leq 1$ we must have $\|w'\| \geq \|w\|$. Consider now the unit vector $w$ for which $\| \Zts w\| = \| \Zts^\dagger \|^{-1}$, that is, the unit norm singular vector corresponding to the smallest non-zero singular value of $\Zts$. For this $w$, and its corresponding vector $w'$, we have
$$ \|\Zts^\dagger\|^{-1} = \|\Zts w\| = \|X_T w' \| \geq \sigma_{\min}(X_T) \|w'\| \geq \sigma_{\min}(X_T)\|w\| = \sigma_{\min}(X_T).$$
It follows that
$$ \| (\Zts^T)^\dagger \|^2 = \| \Zts^\dagger \|^2 \leq 1/\sigma_{\min}(X_T)^2  $$
We continue to bound the right factor of product.
$$  \| \Zts^T y \|^2 = \sum_{i \in T \setminus S} (z_i^T y)^2 =   \sum_{i \in T \setminus S}  g(S)-g(S \cup \{i\}).$$
By combining the inequalities we obtained the required result:
$$ g(S) - g(T) \leq \left( 1/\sigma_{\min}(X_T)^2 \right)  \sum_{i \in T \setminus S}  g(S)-g(S \cup \{i\}).$$
\end{proof}

\begin{proof}[Proof of Lemma~\ref{lemma:RIP.modu}]
We would like to apply Lemma~\ref{lem:regression_sub} on the design matrix $X$.
%
%
%
The only catch is that the columns
of $X$ may not be bounded by $1$ in norm. To remedy this, let $j$ be the index
of the column with the maximum norm and consider the matrix $\bar{X} =
\frac{1}{\|X_j\|}X$ instead (here, $X_j$ is the $j$-th column of $X$; note that
$X_j = Xe_j$ for the $j$-th standard basis vector $e_j$). Now, for any subset
$S$ of coordinates,
\[ \inf_{w \in \R^S} \|y - \barX w\|^2 = \inf_{w \in \R^S} \|y - X w\|^2.\]

Thus, we conclude that the set function of interest, $g(S) = \inf_{w \in \R^S} \|y - X w\|^2$, is $\alpha$-weakly supermodular for sparsity $k$ for $\alpha =
\max_{S: |S| \leq k} \|\barX_S^\dagger\|_2^2$. For any subset of coordinates $S$
of size at most $k$, let $w$ be a unit norm right singular vector of $\barX_S$
corresponding to the smallest singular value, so that $\|\barX_S^\dagger\|_2 =
\frac{1}{\|\barX_Sw\|}$. But $\frac{1}{\|\barX_Sw\|} =
\frac{\|Xe_j\|}{\|Xw'\|}$, where $w'$ is the vector $w$ extended to all
coordinates by padding with zeros.

Since the restricted condition number of $X$ for sparsity $k$ is bounded by $\kappa$ we conclude that $\frac{\|Xe_j\|}{\|Xw'\|} \leq \kappa$. Since
this bound holds for any subset $S$ of size at most $k$, we conclude that
$\alpha \leq \kappa^2$.
\end{proof}

\begin{proof}[Proof of Lemma~\ref{lemma:greedy-modular}]
By the $\alpha$-weak supermodularity of $g$, we have
\begin{align*}
g(\emptyset) - g(V) & \leq \alpha  \cdot \sum_{j\in V} [g(\emptyset) - g(\{j\})] \\
&\leq \alpha |V| \cdot [(g(\emptyset) - g(V)) - (g(\{j^*\}) - g(V))].
\end{align*}
Rearranging, we get the claimed bounds.
\end{proof}

The following lemma gives a useful property of weakly supermodular functions.
\begin{lemma}
\label{lemma:union-modular}
Let $g(\cdot)$ be a $(k, \alpha)$-weakly supermodular set function and $U$ be a subset with $|U|<k$. Then $g'(S) := g(U \cup S)$ is $(k-|U|, \alpha)$-weakly supermodular.
\end{lemma}
\begin{proof}
	For any two subsets $S \subseteq T$ with $|T| \leq k - |U|$, we have
	\begin{align*}
	g'(S) - g'(T)&= g(U \cup S) - g(U \cup T) \leq \alpha \sum_{j \in (T \cup U) \backslash (S\cup U)} [g(U \cup S) - g(U \cup S \cup \{j\})] \\
			 &\leq  \alpha \sum_{j \in T \backslash S} [g(U \cup S) - g(U \cup S \cup \{j\})] = \alpha \sum_{j \in T \backslash S} [g'(S) - g'(S\cup \{j\})].
	\end{align*}
\end{proof}

\begin{proof}[Proof of Lemma~\ref{lem:online-greedy}]

For $i \in \{0, 1, \ldots, \wk\}$, define the set function $g_b^{(i)}$ as
$g_b^{(i)}(S) = g_b(S \cup V_b^{(i)})$.

First, we analyze the performance of the BEXP algorithms. Fix any $i \in [\wk]$
and consider $\bexp^{(i)}$. Conceptually, for any $j \in [d]$, the loss of
expert $j$ at the end of mini-batch $b$ is $g_b(V_b^{(i-1)} \cup {j})$ (note
that this loss is only evaluated for $j \in U_b^{(i)}$ in the algorithm). To
bound the regret, we need to bound the magnitude of the losses. Note that for
any subset $S$, we have $0 \leq g_b(S) \leq \frac{1}{B} \sum_{t \in \mc{T}_b}
y_t^2 \leq 1$. Thus, the regret guarantee of BEXP (Theorem~\ref{theorem:bexp})
implies that for any $i \in [\wk]$ and any $j \in [d]$, we have
\[
	\E\left[\sum_{b=1}^{T/B} g_b(V_b^{(i-1)} \cup \{j_b^{(i)}\})\right] \leq \sum_{b=1}^{T/B} g_b(V_b^{(i-1)} \cup \{j\}) + 2\sqrt{\tfrac{d\wk\log(d)T}{\sk B}}.
\]
The expectation above is conditioned on the randomness in $V_b^{(i-1)}$, for $b \in [T/B]$. Rewriting the above inequality using the $g^{(i-1)}$ and $g^{(i)}$ functions, and using the fact that $V_b^{(i-1)} \cup \{j_b^{(i)}\} = V_b^{(i)}$,  we get
\begin{equation} \label{eq:bexp-regret}
\E\left[\sum_{b=1}^{T/B} g_b^{(i)}(\emptyset)\right] \leq \sum_{b=1}^{T/B} g_b^{(i-1)}(\{j\}) + 2\sqrt{\tfrac{d\wk\log(d)T}{\sk B}}.
\end{equation}

Next, since we assumed that the sequence of feature vectors satisfies BBRCNP
with parameters $(\epsilon, \wk + k)$, Lemma~\ref{lemma:RIP.modu} implies that
the set function $g_b$ defined in \eqref{equation:gt} is $(\wk + k,
\kappa^2)$-weakly supermodular for $\kappa = \frac{1+\epsilon}{1-\epsilon}$. By Lemma~\ref{lemma:union-modular}, the set function $g_b^{(i)}$ is $(k,
\kappa^2)$-weakly supermodular (since $|V_b^{(i)}| \leq \wk$).

It is easy to check that the sum of weakly supermodular functions is also weakly supermodular (with the same parameters), and hence $\sum_{b = 1}^{T/B}g_b^{(i-1)}$ is also $(k, \kappa^2)$-weakly supermodular. Hence, by Lemma~\ref{lemma:greedy-modular}, if $j^* = \argmin_j \sum_{b = 1}^{T/B}g_b^{(i-1)}(\{j\})$, we have, for any subset $V$ of size at most $k$,
\[ \sum_{b = 1}^{T/B}g_b^{(i-1)}(\{j^*\}) - g_b^{(i-1)}(V) \leq (1 - \tfrac{1}{\kappa^2|V|})[\sum_{b = 1}^{T/B}g_b^{(i-1)}(\emptyset) - g_b^{(i-1)}(V)].\]
Since $g_b(V) \geq g_b(V \cup V_b^{(i-1)}) = g_b^{(i-1)}(V)$, the above inequality implies that
\[ \sum_{b = 1}^{T/B}g_b^{(i-1)}(\{j^*\}) - g_b(V) \leq (1 - \tfrac{1}{\kappa^2|V|})[\sum_{b = 1}^{T/B}g_b^{(i-1)}(\emptyset) - g_b(V)].\]
Combining this bound with \eqref{eq:bexp-regret} for $j = j^*$, we get
\[
\E\left[\sum_{b=1}^{T/B} g_b^{(i)}(\emptyset) - g_b(V)\right] \leq (1 - \tfrac{1}{\kappa^2|V|})[\sum_{b = 1}^{T/B}g_b^{(i-1)}(\emptyset) - g_b(V)] + 2\sqrt{\tfrac{d\wk\log(d)T}{\sk B}}.
\]
Applying this bound recursively for $i \in [\wk]$ and simplifying, we get
\[
\E\left[\sum_{b=1}^{T/B} g_b^{(\wk)}(\emptyset) - g_b(V)\right] \leq (1 - \tfrac{1}{\kappa^2|V|})^\wk[\sum_{b = 1}^{T/B}g_b^{(0)}(\emptyset) - g_b(V)] + 2\kappa^2|V|\sqrt{\tfrac{d\wk\log(d) T}{\sk B}}.
\]
Using the definitions of $g_b^{(\wk)}$ and $g_b^{(0)}$, and the fact that $|V| \leq k$, we get the claimed bound.
\end{proof}

\end{document}